\documentclass{article}

\usepackage[final]{nips_2018}

\usepackage[utf8]{inputenc}

\usepackage{amsmath}
\usepackage{amssymb}
\usepackage{amsthm}
\usepackage{microtype}
\usepackage{hyperref}

\usepackage{tikz, graphicx}
\usepackage{subcaption}

\usepackage{natbib}
\setcitestyle{numbers}

\newtheorem{prop}{Proposition}
\newtheorem{thm}{Theorem}
\newtheorem{rem}{Remark}
\newtheorem{lem}{Lemma}

\begin{document}

\title{Neural Tangent Kernel: \\ Convergence and Generalization in Neural Networks}
\author{Arthur Jacot \\ École Polytechnique Fédérale de Lausanne \\ 
\texttt{arthur.jacot@netopera.net}
\And
Franck Gabriel \\ Imperial College London and École Polytechnique Fédérale de Lausanne \\ \texttt{franckrgabriel@gmail.com}
\And
Cl\'ement Hongler \\ École Polytechnique Fédérale de Lausanne \\ \texttt{clement.hongler@gmail.com}
}
\date{May 2018}

\maketitle

\begin{abstract}
At initialization, artificial neural networks (ANNs) are equivalent to Gaussian processes in the infinite-width limit \cite{Neal1996,Daniely, Matthews2017GaussianProcess,Lee2017,Matthews2018GaussianProcess}, thus connecting them to kernel methods. We prove that the evolution of an ANN during training can also be described by a kernel: during gradient descent on the parameters of an ANN, the network function $f_\theta$ (which maps input vectors to output vectors) follows the kernel gradient of the functional cost (which is convex, in contrast to the parameter cost) w.r.t. a new kernel: the Neural Tangent Kernel (NTK). This kernel is central to describe the generalization features of ANNs. While the NTK is random at initialization and varies during training, in the infinite-width limit it converges to an explicit limiting kernel and it stays constant during training. This makes it possible to study the training of ANNs in function space instead of parameter space. Convergence of the training can then be related to the positive-definiteness of the limiting NTK. We prove the positive-definiteness of the limiting NTK when the data is supported on the sphere and the non-linearity is non-polynomial.

We then focus on the setting of least-squares regression and show that in the infinite-width limit, the network function $f_\theta$ follows a linear differential equation during training. The convergence is fastest along the largest kernel principal components of the input data with respect to the NTK, hence suggesting a theoretical motivation for early stopping.

Finally we study the NTK numerically, observe its behavior for wide networks, and compare it to the infinite-width limit.
\end{abstract}
\section{Introduction}
Artificial neural networks (ANNs) have achieved impressive results in numerous areas of machine learning. While it has long been known that ANNs can approximate any function with sufficiently many hidden neurons \cite{Hornik1989, Leshno}, it is not known what the optimization of ANNs converges to. Indeed the loss surface of neural networks optimization problems is highly non-convex: it has a high number of saddle points which may slow down the convergence \citep{Dauphin2014}. A number of results \citep{Choromanska, Pascanu2014, Pennington2017} suggest that for wide enough networks, there are very few ``bad'' local minima, i.e. local minima with much higher cost than the global minimum. More recently, the investigation of the geometry of the loss landscape at initialization has been the subject of a precise study \citep{Karakida2018}. The analysis of the dynamics of training in the large-width limit for shallow networks has seen recent progress as well \citep{Mei2018}. To the best of the authors knowledge, the dynamics of deep networks has however remained an open problem until the present paper: see the contributions section below.

A particularly mysterious feature of ANNs is their good generalization properties in spite of their usual over-parametrization \citep{Sagun}. It seems paradoxical that a reasonably large neural network can fit random labels, while still obtaining good test accuracy when trained on real data \citep{Zhang}. It can be noted that in this case, kernel methods have the same properties \citep{Belkin}.

In the infinite-width limit, ANNs have a Gaussian distribution described by a kernel \cite{Neal1996,Daniely, Matthews2017GaussianProcess,Lee2017,Matthews2018GaussianProcess}. These kernels are used in Bayesian inference or 
Support Vector Machines, yielding results comparable to ANNs trained with gradient descent \cite{Cho2009, Lee2017}. We will see that in the same limit, the behavior of ANNs during training is described by a related kernel, which we call the neural tangent network (NTK).

\subsection{Contribution}
We study the network function $f_\theta$ of an ANN, which maps an input vector to an output vector, where $\theta$ is the vector of the parameters of the ANN. In the limit as the widths of the hidden layers tend to infinity, the network function at initialization, $f_\theta$ converges to a Gaussian distribution \cite{Neal1996,Daniely, Matthews2017GaussianProcess,Lee2017,Matthews2018GaussianProcess}.

In this paper, we investigate fully connected networks in this infinite-width limit, and describe the dynamics of the network function $f_\theta$ during training:
\begin{itemize}
\item
During gradient descent, we show that the dynamics of $f_\theta$ follows that of the so-called \emph{kernel gradient descent} in function space with respect to a limiting kernel, which only depends on the depth of the network, the choice of nonlinearity and the initialization variance.
\item
The convergence properties of ANNs during training can then be related to the positive-definiteness of the infinite-width limit NTK. In the case when the dataset is supported on a sphere, we prove this positive-definiteness using recent results on dual activation functions \cite{Daniely}. The values of the network function $f_\theta$ outside the training set is described by the NTK, which is crucial to understand how ANN generalize.
\item
For a least-squares regression loss, the network function $f_{\theta}$ follows a linear differential equation in the infinite-width limit, and the eigenfunctions of the Jacobian are the kernel principal components of the input data. This shows a direct connection to kernel methods and motivates the use of early stopping to reduce overfitting in the training of ANNs.
\item
Finally we investigate these theoretical results numerically for an artificial dataset (of points on the unit circle) and for the MNIST dataset. In particular we observe that the behavior of wide ANNs is close to the theoretical limit.
\end{itemize}

\section{Neural networks}\label{sec:realization_function}
In this article, we consider fully-connected ANNs with layers numbered from $0$ (input) to $L$ (output), each containing $n_{0},\ldots,n_{L}$ neurons, and with a Lipschitz, twice differentiable nonlinearity function $\sigma:\mathbb{R}\to\mathbb{R}$, with bounded second derivative \footnote{While these smoothness assumptions greatly simplify the proofs of our results, they do not seem to be strictly needed for the results to hold true.}.

This paper focuses on the ANN \emph{realization function} $F^{(L)} : \mathbb{R}^{P} \to \mathcal{F}$, mapping parameters $\theta$ to functions $f_\theta$ in a space $\mathcal{F}$. The dimension of the parameter space is $P = \sum_{\ell=0}^{L-1} (n_\ell + 1) n_{\ell+1}$: the parameters consist of the connection matrices $W^{(\ell)} \in \mathbb{R}^{n_\ell \times n_{\ell+1}}$ and bias vectors $b^{(\ell)} \in \mathbb{R}^{n_{\ell+1}}$ for $\ell=0, ..., L-1$.  In our setup, the parameters are initialized as iid Gaussians $\mathcal{N}(0, 1)$.

For a fixed distribution $p^{in}$ on the input space $\mathbb{R}^{n_0}$, the function space $\mathcal{F}$ is defined as $\left\{ f : \mathbb{R}^{n_0} \to \mathbb{R}^{n_L} \right\} $. On this space, we consider the seminorm $|| \cdot ||_{p^{in}} $, defined in terms of the bilinear form
$$
\left<f, g\right>_{p^{in}} = \mathbb{E}_{x \sim p^{in}}\left[ f(x)^T g(x) \right].
$$

In this paper, we assume that the input distribution $p^{in}$ is the empirical distribution on a finite dataset $x_1, ..., x_N$, i.e the sum of Dirac measures $ \frac{1}{N} \sum_{i=0}^N \delta_{x_i}$. 

We define the network function by $f_\theta(x) := \tilde{\alpha}^{(L)}(x; \theta)$, where the functions $\tilde{\alpha}^{(\ell)}(\cdot; \theta) : \mathbb{R}^{n_0} \to\mathbb{R}^{n_\ell}$ (called \emph{preactivations}) and $\alpha^{(\ell)}(\cdot; \theta):\mathbb{R}^{n_0} \to\mathbb{R}^{n_\ell}$ (called \emph{activations}) are defined from the $0$-th to the $L$-th layer by:
\begin{align*}
    \alpha^{(0)}(x; \theta) &= x \\
    \tilde{\alpha}^{(\ell+1)}(x; \theta) &= \frac{1}{\sqrt{n_\ell}}W^{(\ell)} \alpha^{(\ell)}(x; \theta) + \beta b^{(\ell)} \\
    \alpha^{(\ell)}(x; \theta) &= \sigma(\tilde{\alpha}^{(\ell)}(x; \theta)),
\end{align*}
where the nonlinearity $\sigma$ is applied entrywise. The scalar $\beta > 0$ is a parameter which allows us to tune the influence of the bias on the training.

\begin{rem}\label{rem:parametrization}
Our definition of the realization function $F^{(L)}$ slightly differs from the classical one. Usually, the factors $\frac{1}{\sqrt{n_\ell}}$ and the parameter $\beta$ are absent and the parameters are initialized using what is sometimes called LeCun initialization, taking $W^{(\ell)}_{ij} \sim \mathcal{N}(0, \frac{1}{n_\ell})$ and $b^{(\ell)}_{j} \sim \mathcal{N}(0, 1)$ (or sometimes $b^{(\ell)}_{j} = 0$) to compensate. While the set of representable functions $F^{(L)}(\mathbb{R}^P)$ is the same for both parametrizations (with or without the factors $\frac{1}{\sqrt{n_\ell}}$ and $\beta$), the derivatives of the realization function with respect to the connections $\partial_{W_{ij}^{(\ell)}} F^{(L)}$ and bias $\partial_{b_{j}^{(\ell)}} F^{(L)}$ are scaled by $\frac{1}{\sqrt{n_\ell}}$ and $\beta$ respectively in comparison to the classical parametrization.

The factors $\frac{1}{\sqrt{n_\ell}}$ are key to obtaining a consistent asymptotic behavior of neural networks as the widths of the hidden layers $n_1, ..., n_{L-1}$ grow to infinity. However a side-effect of these factors is that they reduce greatly the influence of the connection weights during training when $n_\ell$ is large: the factor $\beta$ is introduced to balance the influence of the bias and connection weights. In our numerical experiments, we take $\beta=0.1$ and use a learning rate of $1.0$, which is larger than usual, see Section \ref{sec:numerical-experiments}. This gives a behaviour similar to that of a classical network of width $100$ with a learning rate of $0.01$.
\end{rem}

\section{Kernel gradient}\label{sec:kernel_gradient}
The training of an ANN consists in optimizing $f_\theta$ in the function space $\mathcal{F}$ with respect to a functional cost $C : \mathcal{F} \to \mathbb{R}$, such as a regression or cross-entropy cost. Even for a convex functional cost $C$, the composite cost $C \circ F^{(L)} : \mathbb{R}^P \to \mathbb{R}$ is in general highly non-convex \citep{Choromanska}. We will show that during training, the network function $f_\theta$ follows a descent along the kernel gradient with respect to the Neural Tangent Kernel (NTK) which we introduce in Section \ref{sec:neural_tangent_kernel}. This makes it possible to study the training of ANNs in the function space $\mathcal{F}$, on which the cost $C$ is convex.

A \emph{multi-dimensional kernel}  $K$ is a function $\mathbb{R}^{n_{0}}\times\mathbb{R}^{n_{0}}\to\mathbb{R}^{n_{L}\times n_{L}}$, which maps any pair $(x,x')$ to an $n_{L}\times n_{L}$-matrix such that $K(x, x') = K(x', x)^T$ (equivalently $K$ is a symmetric tensor in $\mathcal{F}\otimes\mathcal{F}$). Such a kernel defines a bilinear map on $\mathcal{F}$, taking the expectation over independent $x, x' \sim p^{in}$:
$$
\left<f, g \right>_K := \mathbb{E}_{x, x' \sim p^{in}} \left[f(x)^T K(x, x') g(x') \right].
$$
The kernel $K$ is \emph{positive definite with respect to  $|| \cdot ||_{p^{in}} $}  if $ || f ||_{p^{in}} > 0 \implies || f ||_K > 0$.

We denote by $ \mathcal{F}^{*} $ the dual of $ \mathcal{F} $ with respect to $ p^{in} $, i.e. the set of linear forms $ \mu: \mathcal {F} \to \mathbb{R} $ of the form $ \mu = \langle d, \cdot \rangle_{p^{in}} $ for some $ d \in \mathcal{F} $.  Two elements of $\mathcal{F}$ define the same linear form if and only if they are equal on the data. The constructions in the paper do not depend on the element $d \in \mathcal{F}$ chosen in order to represent $ \mu$ as $\langle d, \cdot \rangle_{p^{in}}$. Using the fact that the partial application of the kernel $K_{i, \cdot}(x, \cdot)$ is a function in $\mathcal{F}$, we can define a map $\Phi_K : \mathcal{F}^* \to \mathcal{F}$ mapping a dual element $\mu = \left<d, \cdot\right>_{p^{in}}$ to the function $f_\mu = \Phi_K(\mu)$ with values:
$$
f_{\mu, i}(x) = \mu K_{i, \cdot}(x, \cdot) =  \left< d, K_{i, \cdot}(x, \cdot) \right>_{p^{in}}.
$$

For our setup, which is that of a finite dataset $ x_1, \ldots, x_n \in \mathbb{R}^{n_0} $, the cost functional $ C $ only depends on the values of $ f \in \mathcal{F} $ at the data points. As a result, the (functional) derivative of the cost $ C $ at a point $f_0\in\mathcal{F}$ can be viewed as an element of $ \mathcal{F}^{*} $, which we write $ \partial_f^{in} C |_{f_0} $. We denote by  $ d |_{f_0} \in \mathcal{F} $, a corresponding dual element, such that $ \partial_f^{in} C|_{f_0}  = \langle d |_{f_0},  \cdot \rangle_{p^{in}} $.

The \emph{kernel gradient} $\nabla_K C|_{f_0} \in \mathcal{F}$ is defined as $\Phi_K \left( \partial_f^{in} C|_{f_0} \right) $. In contrast to $ \partial_f^{in} C $ which is only defined on the dataset,  the kernel gradient generalizes to values $x$ outside the dataset thanks to the kernel $K$: 
\[
	\nabla_K C|_{f_0} (x) = \frac{1}{N} \sum_{j=1}^{N} K(x, x_j) d|_{f_0} (x_j) .
\] 

A time-dependent function $ f (t) $ follows the \emph{kernel gradient descent with respect to $K$} if it satisfies the differential equation  
\[
	\partial_t f(t) = -\nabla_{K}C|_{f(t)}.
\] During kernel gradient descent, the cost $C(f(t))$ evolves as
\[
	\partial_t C|_{f(t)} = -\left<d|_{f(t)}, \nabla_{K}C|_{f(t)} \right>_{p^{in}}=-\left\|d|_{f(t)} \right\|_{K}^{2}.
\]
Convergence to a critical point of $C$ is hence guaranteed if the kernel $K$ is positive definite with respect to $|| \cdot ||_{p^{in}}$: the cost is then strictly decreasing except at points such that $||d|_{f(t)}||_{p^{in}} = 0$. If the cost is convex and bounded from below, the function $f(t)$ therefore converges to a global minimum as $t \to \infty$.

\subsection{Random functions approximation} \label{sec:rfa}
As a starting point to understand the convergence of ANN gradient descent to kernel gradient descent in the infinite-width limit, we introduce a simple model, inspired by the approach of \cite{Rahimi2007}. 

A kernel $K$ can be approximated by a choice of $P$ random functions $f^{(p)}$ sampled independently from any distribution on $\mathcal{F}$ whose (non-centered) covariance is given by the kernel $K$:
\[
\mathbb{E}[f_{k}^{(p)}(x)f_{k'}^{(p)}(x')]=K_{kk'}(x,x').
\]

These functions define a random linear parametrization $F^{lin}:\mathbb{R}^{P}\to\mathcal{F}$
\[
\theta \mapsto f_{\theta}^{lin}=\frac{1}{\sqrt{P}}\sum_{p=1}^{P}\theta_{p}f^{(p)}.
\]

The partial derivatives of the parametrization are given by 
\[
	\partial_{\theta_p} F^{lin}(\theta) = \frac{1}{\sqrt{P}} f^{(p)}.
\] 

Optimizing the cost $C\circ F^{lin}$ through gradient descent, the parameters follow the ODE:
\[
\partial_t  \theta_{p}(t)=-\partial_{\theta_{p}}(C\circ F^{lin})(\theta(t))=-\frac{1}{\sqrt{P}}\partial_{f}^{in}C|_{f_{\theta(t)}^{lin}}\;f^{(p)}=-\frac{1}{\sqrt{P}} \left<d|_{f_{\theta(t)}^{lin}}, f^{(p)}\right>_{p^{in}}.
\]
As a result the function $f_{\theta(t)}^{lin}$ evolves according to
\[
\partial_t f_{\theta(t)}^{lin}=\frac{1}{\sqrt{P}}\sum_{p=1}^{P} \partial_t \theta_{p}(t)f^{(p)}=-\frac{1}{P}\sum_{p=1}^{P}\left<d|_{f_{\theta(t)}^{lin}}, f^{(p)}\right>_{p^{in}} f^{(p)},
\]
where the right-hand side is equal to the kernel gradient $-\nabla_{\tilde{K}} C$ with respect to the \emph{tangent kernel} 
$$
\tilde{K} =\sum_{p=1}^P \partial_{\theta_p} F^{lin}(\theta) \otimes  \partial_{\theta_p} F^{lin}(\theta) = \frac{1}{P}\sum_{p=1}^P f^{(p)} \otimes f^{(p)}.
$$ 
This is a random $n_L$-dimensional kernel with values
$
\tilde{K}_{ii'}(x, x') = \frac{1}{P} \sum_{p=1}^P  f^{(p)}_i(x) f^{(p)}_{i'}(x').
$

Performing gradient descent on the cost $ C \circ F^{lin}$ is therefore equivalent to performing kernel gradient descent with the tangent kernel $\tilde{K}$ in the function space. In the limit as $P \to \infty$, by the law of large numbers, the (random) tangent kernel $\tilde{K}$ tends to the fixed kernel $K$, which makes this method an approximation of kernel gradient descent with respect to the limiting kernel $K$.

\section{Neural tangent kernel} \label{sec:neural_tangent_kernel}
For ANNs trained using gradient descent on the composition $C \circ F^{(L)}$, the situation is very similar to that studied in the Section \ref{sec:rfa}. During training, the network function $f_\theta$ evolves along the (negative) kernel gradient
$$
\partial_t  f_{\theta(t)} = -\nabla_{\Theta^{(L)}} C|_{f_{\theta(t)}}
$$
with respect to the \emph{neural tangent kernel} (NTK)
\begin{align*}
 \Theta^{(L)}(\theta) &= \sum_{p=1}^P \partial_{\theta_p} F^{(L)}(\theta) \otimes \partial_{\theta_p} F^{(L)}(\theta).
\end{align*}
However, in contrast to $F^{lin}$, the realization function $F^{(L)}$ of ANNs is not linear. As a consequence, the derivatives $\partial_{\theta_p} F^{(L)}(\theta)$ and the neural tangent kernel depend on the parameters $\theta$. The NTK is therefore random at initialization and varies during training, which makes the analysis of the convergence of $f_\theta$ more delicate.

In the next subsections, we show that, in the infinite-width limit, the NTK becomes deterministic at initialization and stays constant during training. Since $f_\theta$ at initialization is Gaussian in the limit, the asymptotic behavior of $f_\theta$ during training can be explicited in the function space $ \mathcal{F} $.

\subsection{Initialization}
As observed in \cite{Neal1996,Daniely, Matthews2017GaussianProcess,Lee2017,Matthews2018GaussianProcess}, the output functions $f_{\theta, i}$ for $i=1, ..., n_L$ tend to iid Gaussian processes in the infinite-width limit (a proof in our setup is given in the appendix):
\begin{prop}\label{prop:output_limit}
For a network of depth $L$ at initialization, with a Lipschitz nonlinearity $\sigma$, and in the limit as $n_1, ..., n_{L-1} \to \infty$, the output functions $f_{\theta, k}$, for $k=1, ..., n_L$, tend (in law) to iid centered Gaussian processes of covariance $\Sigma^{(L)}$, where $\Sigma^{(L)}$ is defined recursively by:
\begin{align*}
\Sigma^{(1)}(x, x') &= \frac{1}{n_0} x^T x' + \beta^2 \\
\Sigma^{(L+1)}(x, x') &= \mathbb{E}_{f\sim\mathcal{N}\left(0,\Sigma^{\left(L\right)}\right)}[\sigma(f(x)) \sigma(f(x'))] + \beta^2,
\end{align*}
taking the expectation with respect to a centered Gaussian process $f$ of covariance $\Sigma^{(L)}$.
\end{prop}

\begin{rem} \label{rem:no-problem-with-gauss-meas}
Strictly speaking, the existence of a suitable Gaussian measure with covariance $\Sigma^{(L)}$ is not needed: we only deal with the values of $f$ at $x, x'$ (the joint measure on $f(x), f(x')$ is simply a Gaussian vector in 2D). For the same reasons, in the proof of Proposition \ref{prop:output_limit} and Theorem \ref{thm:convergence_NTK_initialization}, we will freely speak of Gaussian processes without discussing their existence.
\end{rem}

The first key result of our paper (proven in the appendix) is the following: in the same limit, the Neural Tangent Kernel (NTK) converges in probability to an explicit deterministic limit. 
\begin{thm}\label{thm:convergence_NTK_initialization}
For a network of depth $L$ at initialization, with a Lipschitz nonlinearity $\sigma$, and in the limit as the layers width $n_1, ..., n_{L-1} \to \infty$, the NTK $\Theta^{(L)}$ converges in probability to a deterministic limiting kernel: $$\Theta^{(L)} \to \Theta^{(L)}_\infty \otimes Id_{n_L}.$$
The scalar kernel $\Theta^{(L)}_\infty : \mathbb{R}^{n_0} \times \mathbb{R}^{n_0} \to \mathbb{R}$ is defined recursively by
\begin{align*}
    \Theta^{(1)}_\infty(x, x') &= \Sigma^{(1)}(x, x') \\
    \Theta^{(L+1)}_\infty(x, x') &=  \Theta^{(L)}_\infty(x, x') \dot{\Sigma}^{(L+1)}(x, x')  + \Sigma^{(L+1)}(x, x'),
\end{align*}
where
\[
	\dot{\Sigma}^{(L+1)}\left(x,x'\right)=
	\mathbb{E}_{f\sim\mathcal{N}\left(0,\Sigma^{\left(L\right)}\right)}\left[\dot{\sigma}\left(f\left(x\right)\right)\dot{\sigma}\left(f\left(x'\right)\right)\right],
\]
taking the expectation with respect to a centered Gaussian process $f$ of covariance $\Sigma^{(L)}$, and where $\dot{\sigma}$ denotes the derivative of $\sigma$.
\end{thm}
\begin{rem}
By Rademacher's theorem, $\dot{\sigma}$ is defined everywhere, except perhaps on a set of zero Lebesgue measure.
\end{rem}

Note that the limiting $\Theta^{(L)}_\infty$ only depends on the choice of $\sigma$, the depth of the network and the variance of the parameters at initialization (which is equal to $1$ in our setting).

\subsection{Training}
Our second key result is that the NTK stays asymptotically constant during training. This applies for a slightly more general definition of training: the parameters are updated according to a training direction $d_t \in \mathcal{F}$:
$$
\partial_t \theta_p(t) = \left< \partial_{\theta_p} F^{(L)}(\theta(t)), d_t \right>_{p^{in}}.
$$
In the case of gradient descent, $d_t = -d|_{f_{\theta(t)}}$ (see Section \ref{sec:kernel_gradient}), but the direction may depend on another network, as is the case for e.g. Generative Adversarial Networks \cite{Goodfellow2014}. We only assume that the integral $\int_0^T \| d_t \|_{p^{in}} dt$ stays stochastically bounded as the width tends to infinity, which is verified for e.g. least-squares regression, see Section \ref{sec:least-squares}.

\begin{thm}\label{thm:conv-ntk-training}

Assume that $ \sigma $ is a Lipschitz, twice differentiable nonlinearity function, with bounded second derivative. For any $T$ such that the integral $\int_0^T \| d_t \|_{p^{in}} dt$ stays stochastically bounded, as  $n_1, ..., n_{L-1} \to \infty$, we have, uniformly for $t\in[0, T]$,
$$
\Theta^{(L)}(t) \to \Theta^{(L)}_\infty \otimes Id_{n_L}.
$$
As a consequence, in this limit, the dynamics of $f_\theta$ is described by the differential equation
\begin{align*}
    \partial_t f_{\theta(t)} = \Phi_{\Theta^{(L)}_\infty \otimes Id_{n_L}} \left( \left<d_{t}, \cdot \right>_{p^{in}} \right).
\end{align*}
\end{thm}

\begin{rem}
As the proof of the theorem (in the appendix) shows, the variation during training of the individual activations in the hidden layers shrinks as their width grows. However their collective variation is significant, which allows the parameters of the lower layers to learn: in the formula of the limiting NTK  $\Theta^{(L+1)}_\infty(x, x')$ in Theorem \ref{thm:convergence_NTK_initialization}, the second summand $\Sigma^{(L+1)}$ represents the learning due to the last layer, while the first summand represents the learning performed by the lower layers. 
\end{rem}

As discussed in Section \ref{sec:kernel_gradient}, the convergence of kernel gradient descent to a critical point of the cost $C$ is guaranteed for positive definite kernels. The limiting NTK is positive definite if the span of the derivatives $\partial_{\theta_p} F^{(L)}$, $p=1, ..., P$ becomes dense in $\mathcal{F}$ w.r.t. the $p^{in}$-norm as the width grows to infinity. It seems natural to postulate that the span of the preactivations of the last layer (which themselves appear in $\partial_{\theta_p} F^{(L)}$, corresponding to the connection weights of the last layer) becomes dense in $\mathcal{F}$, for a large family of measures $ p^{in}$ and nonlinearities (see e.g. \cite{Hornik1989, Leshno} for classical theorems about ANNs and approximation). In the case when the dataset is supported on a sphere, the positive-definiteness of the limiting NTK can be shown using Gaussian integration techniques and existing positive-definiteness criteria, as given by the following proposition, proven in Appendix \ref{Appendix-4}:
\begin{prop}\label{prop:pos-def}
	 For a non-polynomial Lipschitz nonlinearity $ \sigma $, for any input dimension $ n_0 $, the restriction of the limiting NTK $ \Theta_\infty^{(L)} $ to the unit sphere $ \mathbb{S}^{n_0 - 1} = \{ x \in \mathbb{R}^{n_0} : x^T x =1 \} $ is positive-definite if $ L \geq 2 $. 
\end{prop}

\section{Least-squares regression}\label{sec:least-squares}
Given a goal function $f^*$ and input distribution $p^{in}$, the least-squares regression cost is
$$
 C(f) = \frac{1}{2} ||f - f^*||^2_{p^{in}} = \frac{1}{2}\mathbb{E}_{x \sim p^{in}} \left[\|f(x) - f^*(x)\|^2 \right].
$$
Theorems  \ref{thm:convergence_NTK_initialization} and \ref{thm:conv-ntk-training} apply to an ANN trained on such a cost. Indeed the norm of the training direction $\|d(f)\|_{p^{in}} = \| f^* - f \|_{p^{in}}$ is strictly decreasing during training, bounding the integral. We are therefore interested in the behavior of a function $f_t$ during kernel gradient descent with a kernel $K$ (we are of course especially interested in the case $K = \Theta^{(L)}_\infty \otimes Id_{n_L}$):
\begin{align*}
    \partial_t f_t = \Phi_K\left(\left<f^* - f, \cdot \right>_{p^{in}}\right).
\end{align*}
The solution of this differential equation can be expressed in terms of the map $\Pi : f \mapsto \Phi_K \left(\left<f, \cdot \right>_{p^{in}}\right)$:
$$
 f_t = f^* + e^{-t \Pi}(f_0 - f^*)
$$
where $e^{-t \Pi} = \sum_{k=0}^{\infty} \frac{(-t)^k}{k!} \Pi^k$ is the exponential of $-t \Pi$. If $\Pi$ can be diagonalized by eigenfunctions $f^{(i)}$ with eigenvalues $\lambda_i$, the exponential $e^{-t \Pi}$ has the same eigenfunctions with eigenvalues $e^{-t \lambda_i}$.

For a finite dataset $x_1, ..., x_N$ of size $N$, the map $\Pi$ takes the form
$$
\Pi(f)_k (x) = \frac{1}{N} \sum_{i=1}^N \sum_{k'=1}^{n_L} f_{k'}(x_i) K_{kk'}(x_i, x).
$$
The map $\Pi$ has at most $Nn_L$ positive eigenfunctions, and they are the kernel principal components $f^{(1)}, ..., f^{(N n_L)}$ of the data with respect to to the kernel $K$ \cite{Scholkopf, Shawe-Taylor}. The corresponding eigenvalues $\lambda_i$ is the variance captured by the component.

Decomposing the difference $(f^* - f_0) = \Delta^0_f + \Delta^1_f + ... + \Delta^{N n_L}_f$ along the eigenspaces of $\Pi$, the trajectory of the function $f_t$ reads
$$
 f_{t} = f^* + \Delta^0_f + \sum_{i=1}^{N n_L} e^{-t \lambda_i} \Delta^i_f,
$$
where $\Delta^0_f$ is in the kernel (null-space) of $\Pi$ and $\Delta^i_f \propto f^{(i)}$. 

The above decomposition can be seen as a motivation for the use of early stopping. The convergence is indeed faster along the eigenspaces corresponding to larger eigenvalues $\lambda_i$. Early stopping hence focuses the convergence on the most relevant kernel principal components, while avoiding to fit the ones in eigenspaces with lower eigenvalues (such directions are typically the `noisier' ones: for instance, in the case of the RBF kernel, lower eigenvalues correspond to high frequency functions).

Note that by the linearity of the map $e^{-t \Pi}$, if $f_0$ is initialized with a Gaussian distribution (as is the case for ANNs in the infinite-width limit), then $f_t$ is Gaussian for all times $t$. Assuming that the kernel is positive definite on the data (implying that the $Nn_L \times Nn_L$ Gram marix $\tilde{K}=\left(K_{kk'}(x_i, x_j) \right)_{ik, jk'}$ is invertible), as $t \to \infty$ limit, we get that $f_\infty = f^* + \Delta^0_f = f_0 - \sum_i \Delta^i_f$ takes the form
$$
 f_{\infty, k}(x) = \kappa_{x, k}^T \tilde{K}^{-1} y^* + \left(f_0(x) - \kappa_{x, k}^T \tilde{K}^{-1} y_0\right),
$$
with the $N n_l$-vectors $ \kappa_{x, k} $, $ y^* $ and $ y_0  $ given by
\begin{align*}
\kappa_{x, k} & = \left(K_{kk'}(x, x_i)\right)_{i, k'} \\
y^* & = \left(f^*_k(x_i)\right)_{i, k} \\
y_0 & = \left(f_{0, k}(x_i)\right)_{i, k}.
\end{align*}
The first term, the mean, has an important statistical interpretation: it is the maximum-a-posteriori (MAP) estimate given a Gaussian prior on functions $f_k\sim \mathcal{N}(0, \Theta^{(L)}_\infty)$  and the conditions $f_k(x_i)=f^*_k(x_i)$ . Equivalently, it is equal to the kernel ridge regression \cite{Shawe-Taylor} as the regularization goes to zero ($\lambda \to 0$). The second term is a centered Gaussian whose variance vanishes on the points of the dataset.

\section{Numerical experiments} \label{sec:numerical-experiments}
In the following numerical experiments, fully connected ANNs of various widths are compared to the theoretical infinite-width limit. We choose the size of the hidden layers to all be equal to the same value $n := n_1 = ... = n_{L-1}$ and we take the ReLU nonlinearity $\sigma(x)=\max(0,x)$. 

In the first two experiments, we consider the case $n_0=2$. Moreover, the input elements are taken on the unit circle. This can be motivated by the structure of high-dimensional data, where the centered data points often have roughly the same norm \footnote{The classical example is for data following a Gaussian distribution $\mathcal{N}(0, Id_{n_0})$: as the dimension $n_0$ grows, all data points have approximately the same norm $\sqrt{n_0}$.}.

In all experiments, we took $n_L = 1$ (note that by our results, a network with $n_L$ outputs behaves asymptotically like $n_L$ networks with scalar outputs trained independently). Finally, the value of the parameter $\beta$ is chosen as $0.1$, see Remark \ref{rem:parametrization}.

\begin{figure*}
    \centering
    \begin{minipage}[t]{0.49\textwidth}
        \centering
        \includegraphics[width=1.0\textwidth]{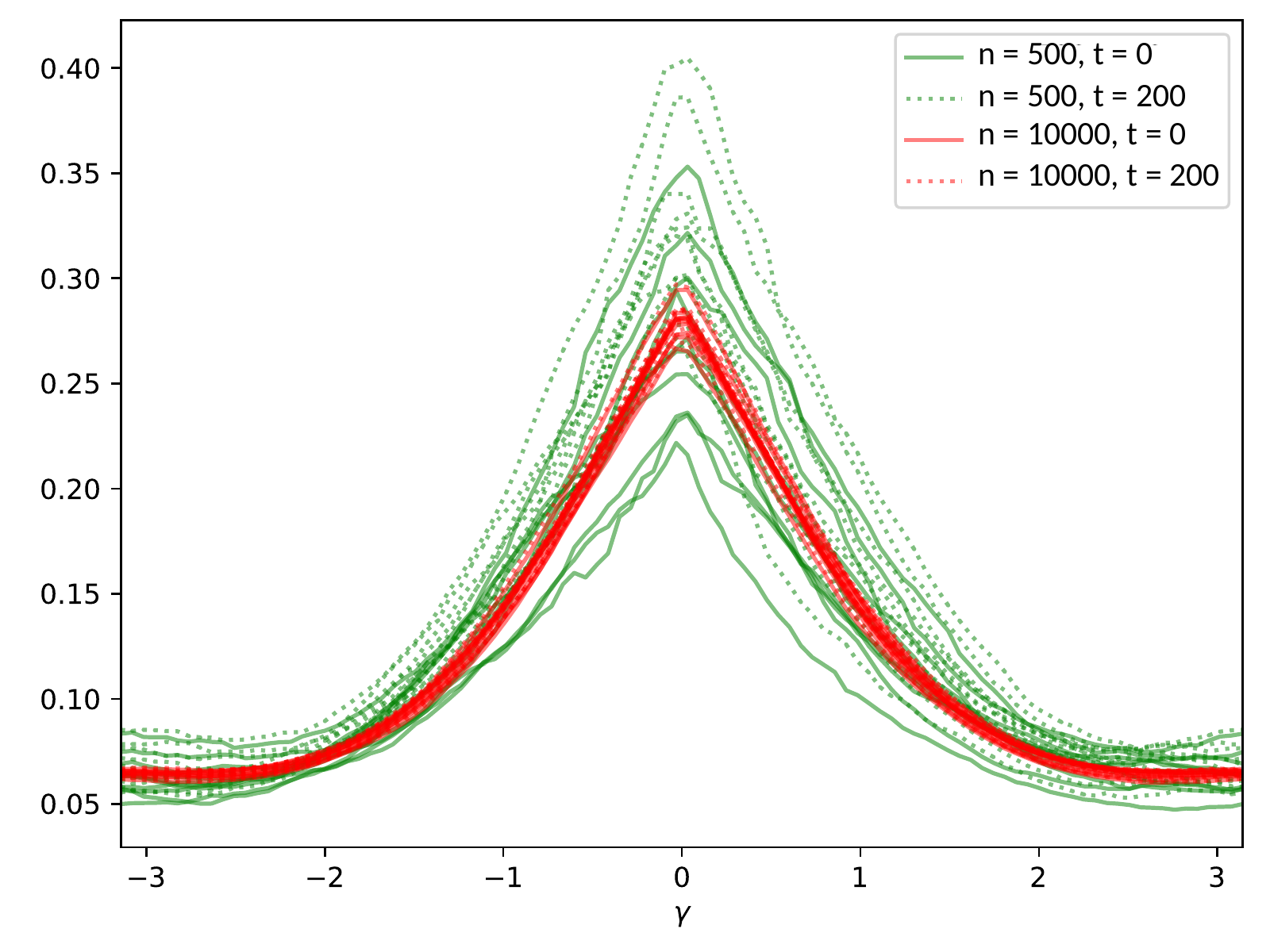}
        \caption{Convergence of the NTK to a fixed limit for two widths $n$ and two times $t$.}
        \label{fig:NTK_convergence}
    \end{minipage}\;
    \begin{minipage}[t]{0.49\textwidth}
        \centering
        \includegraphics[width=1.0\textwidth]{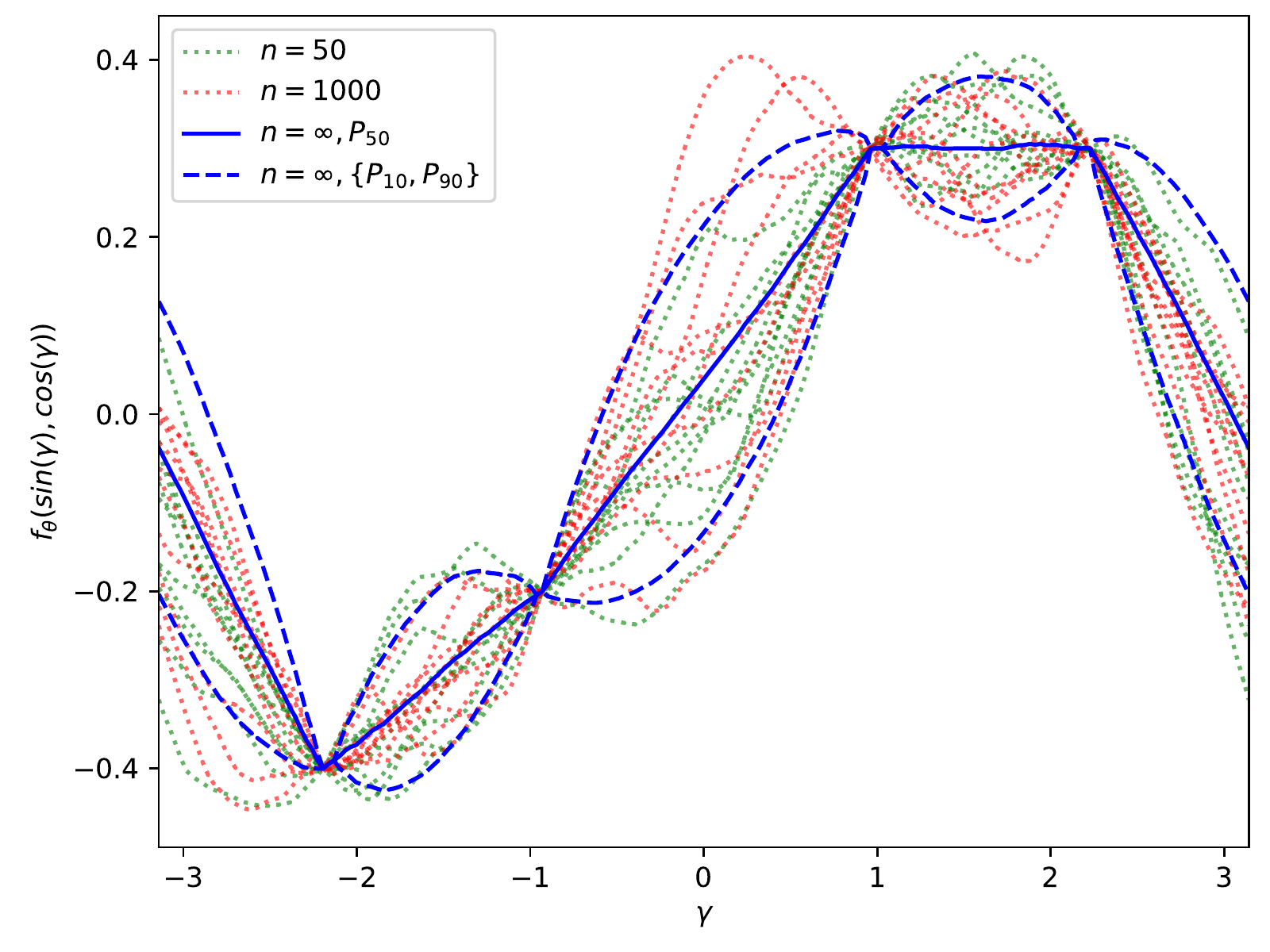}
        \caption{Networks function $f_\theta$ near convergence for two widths $n$ and 10th, 50th and 90th percentiles of the asymptotic Gaussian distribution.}
        \label{fig:ANN_regression}
    \end{minipage}
\end{figure*}

\subsection{Convergence of the NTK}
The first experiment illustrates the convergence of the NTK $\Theta^{(L)}$ of a network of depth $L=4$ for two different widths $n=500, 10000$. The function $\Theta^{(4)}(x_0, x)$ is plotted for a fixed $x_0=(1, 0)$ and $x=(cos(\gamma), sin(\gamma))$ on the unit circle in Figure \ref{fig:NTK_convergence}. To observe the distribution of the NTK, $10$ independent initializations are performed for both widths. The kernels are plotted at initialization $t=0$ and then after $200$ steps of gradient descent with learning rate $1.0$ (i.e. at $t=200$). We approximate the function $f^*(x) = x_1 x_2$ with a least-squares cost on random $\mathcal{N}(0, 1)$ inputs.

For the wider network, the NTK shows less variance and is smoother. It is interesting to note that the expectation of the NTK is very close for both networks widths. After $200$ steps of training, we observe that the NTK tends to ``inflate''. As expected, this effect is much less apparent for the wider network ($n=10000$) where the NTK stays almost fixed, than for the smaller network ($n=500$).

\subsection{Kernel regression}
For a regression cost, the infinite-width limit network function $f_{\theta(t)}$ has a Gaussian distribution for all times $t$ and in particular at convergence $t\to \infty $ (see Section \ref{sec:least-squares}). We compared the theoretical Gaussian distribution at $t\to\infty$ to the distribution of the network function $f_{\theta(T)}$ of a finite-width network for a large time $T=1000$. For two different widths $n=50, 1000$ and for $10$ random initializations each, a network is trained on a least-squares cost on $4$ points of the unit circle for $1000$ steps with learning rate $1.0$ and then plotted in Figure \ref{fig:ANN_regression}.

We also approximated the kernels $\Theta_\infty^{(4)}$ and $\Sigma^{(4)}$ using a large-width network ($n=10000$) and used them to calculate and plot the 10th, 50th and 90-th percentiles of the $t\to \infty$ limiting Gaussian distribution.

The distributions of the network functions are very similar for both widths: their mean and variance appear to be close to those of the limiting distribution $t \to \infty$. Even for relatively small widths ($n=50$), the NTK gives a good indication of the distribution of $f_{\theta(t)}$ as $t\to\infty$.

\subsection{Convergence along a principal component}
We now illustrate our result on the MNIST dataset of handwritten digits made up of grayscale images of dimension $28 \times 28$, yielding a dimension of $n_0 = 784$.

We computed the first 3 principal components of a batch of $N=512$ digits with respect to the NTK of a high-width network $n=10000$ (giving an approximation of the limiting kernel) using a power iteration method. The respective eigenvalues are $\lambda_1=0.0457$, $\lambda_2=0.00108$ and $\lambda_3=0.00078$. The kernel PCA is non-centered, the first component is therefore almost equal to the constant function, which explains the large gap between the first and second eigenvalues\footnote{It can be observed numerically, that if we choose $\beta=1.0$ instead of our recommended $0.1$, the gap between the first and the second principal component is about ten times bigger, which makes training more difficult.}. The next two components are much more interesting as can be seen in Figure \ref{fig:MNIST_kernel_principal_components}, where the batch is plotted with $x$ and $y$ coordinates corresponding to the 2nd and 3rd components.

\begin{figure*}
    \centering
    \begin{subfigure}[t]{0.285\textwidth}
        \includegraphics[width=1.05\textwidth]{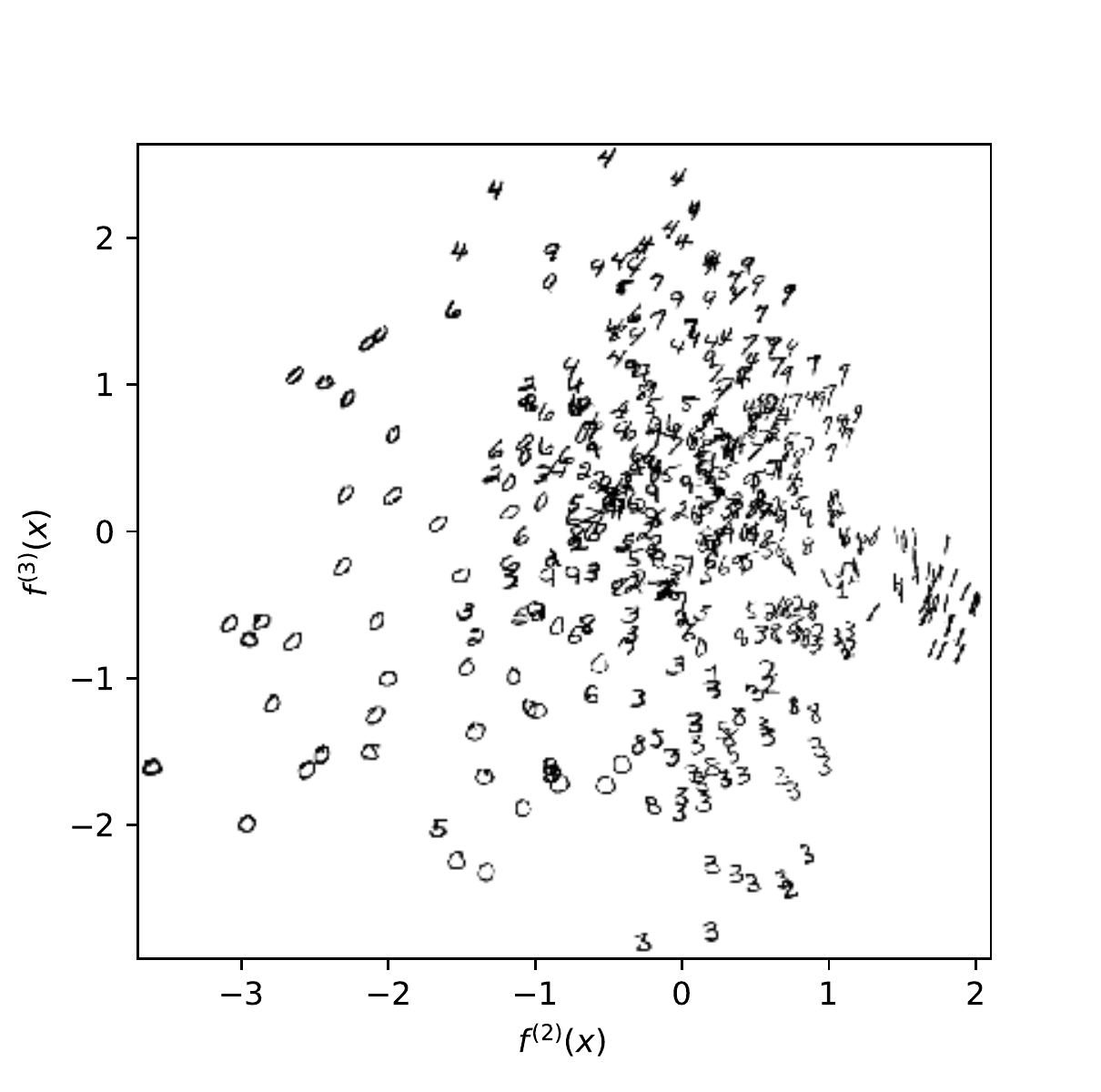}
        \caption{The 2nd and 3rd principal components of MNIST.}
        \label{fig:MNIST_kernel_principal_components}
    \end{subfigure}\;
    \begin{subfigure}[t]{0.34\textwidth}
  	 \begin{picture}(10,9)
   	  \put(5,3){\includegraphics[width=1.0\textwidth]{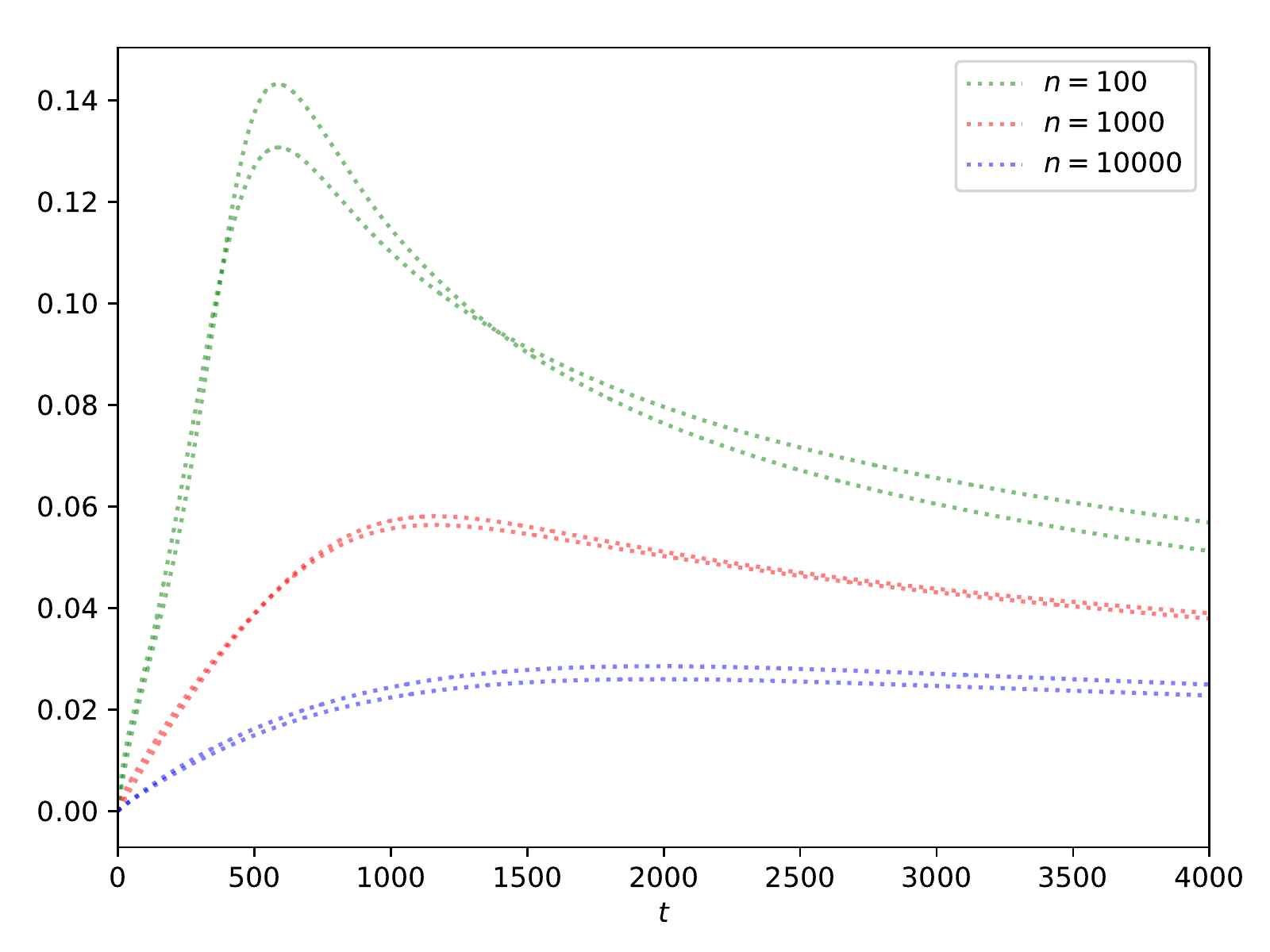}}
   	  \put(0,45){\scalebox{.6}{\rotatebox{90}{$||h_t||_{p^{in}}$ }}}
  	 \end{picture}
        \caption{Deviation of the network function $f_\theta$ from the straight line.}
        \label{fig:MNIST_convergence_outside}
    \end{subfigure}\;
	\begin{subfigure}[t]{0.34\textwidth}
  	 \begin{picture}(10,9)
   	  \put(5,3){\includegraphics[width=1.0\textwidth]{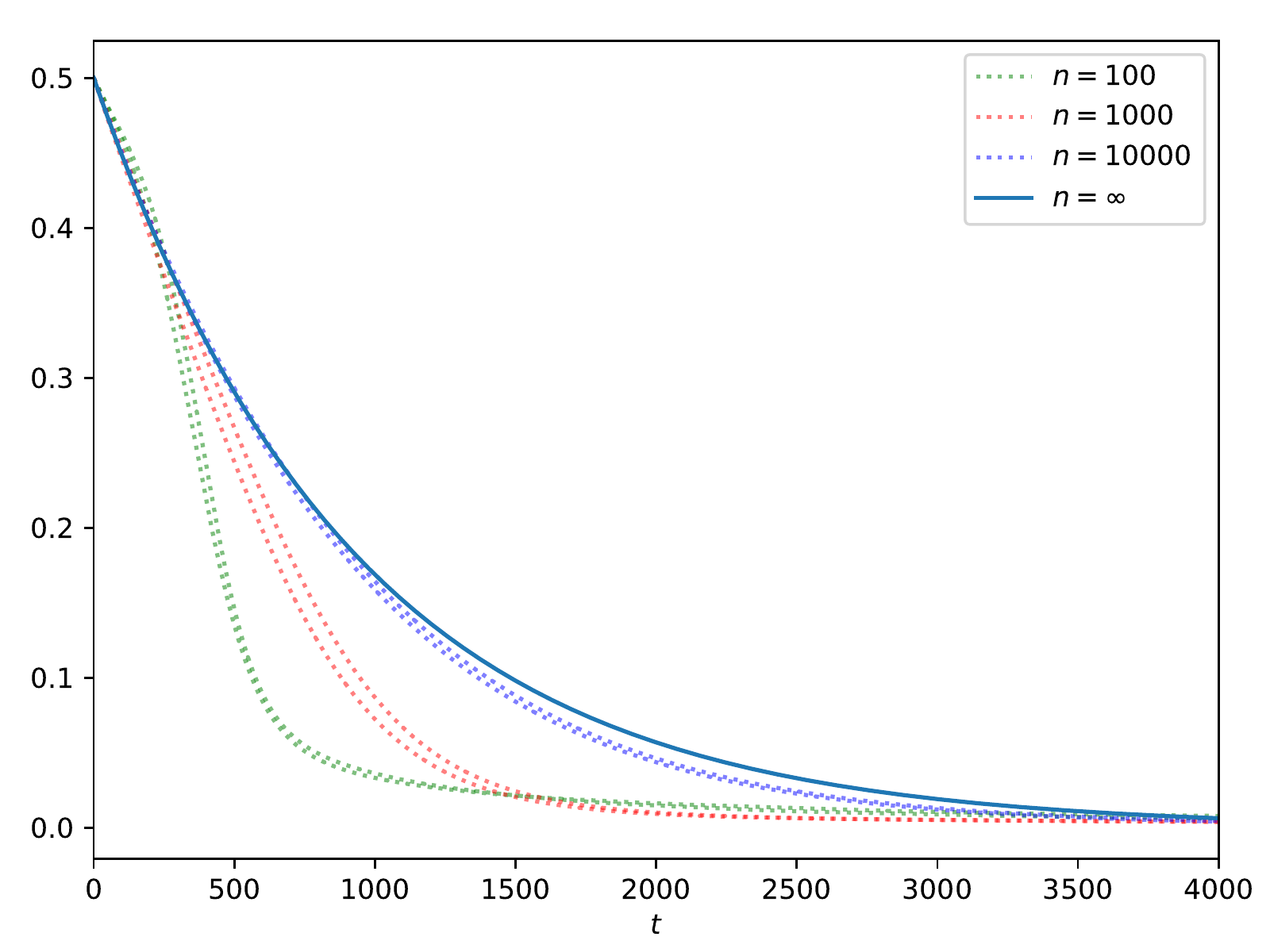}}
   	  \put(0,45){\scalebox{.6}{\rotatebox{90}{$||g_t||_{p^{in}}$ }}}
  	 \end{picture}
        \caption{Convergence of $f_\theta$ along the 2nd principal component.}
        \label{fig:MNIST_convergence_inside}
    \end{subfigure}
    \caption{}
\end{figure*}

We have seen in Section \ref{sec:least-squares} how the convergence of kernel gradient descent follows the kernel principal components. If the difference at initialization $f_0 - f^*$ is equal (or proportional) to one of the principal components $f^{(i)}$, then the function will converge along a straight line (in the function space) to $f^*$ at an exponential rate $e^{-\lambda_i t}$.

We tested whether ANNs of various widths $n=100, 1000, 10000$ behave in a similar manner. We set the goal of the regression cost to $f^* = f_{\theta(0)}+0.5 f^{(2)}$ and let the network converge. At each time step $t$, we decomposed the difference $f_{\theta(t)} - f^*$ into a component $g_t$ proportional to $f^{(2)}$ and another one $h_t$ orthogonal to $f^{(2)}$. In the infinite-width limit, the first component decays exponentially fast $||g_t||_{p^{in}} = 0.5 e^{-\lambda_2 t}$ while the second is null ($h_t=0$), as the function converges along a straight line.

As expected, we see in Figure \ref{fig:MNIST_convergence_outside} that the wider the network, the less it deviates from the straight line (for each width $n$ we performed two independent trials). As the width grows, the trajectory along the 2nd principal component (shown in Figure \ref{fig:MNIST_convergence_inside}) converges to the theoretical limit shown in blue.

A surprising observation is that smaller networks appear to converge faster than wider ones. This may be explained by the inflation of the NTK observed in our first experiment. Indeed, multiplying the NTK by a factor $a$ is equivalent to multiplying the learning rate by the same factor. However, note that since the NTK of large-width network is more stable during training, larger learning rates can in principle be taken. One must hence be careful when comparing the convergence speed in terms of the number of steps (rather than in terms of the time $t$): both the inflation effect and the learning rate must be taken into account.

\section{Conclusion}
This paper introduces a new tool to study ANNs, the Neural Tangent Kernel (NTK), which describes the local dynamics of an ANN during gradient descent. This leads to a new connection between ANN training and kernel methods: in the infinite-width limit, an ANN can be described in the function space directly by the limit of the NTK, an explicit constant kernel $\Theta^{(L)}_\infty$, which only depends on its depth, nonlinearity and parameter initialization variance. More precisely, in this limit, ANN gradient descent is shown to be equivalent to a kernel gradient descent with respect to $\Theta^{(L)}_\infty$. The limit of the NTK is hence a powerful tool to understand the generalization properties of ANNs, and it allows one to study the influence of the depth and nonlinearity on the learning abilities of the network. The analysis of training using NTK allows one to relate convergence of ANN training with the positive-definiteness of the limiting NTK and leads to a characterization of the directions favored by early stopping methods.

\section*{Acknowledgements}
The authors thank K. Kyt\"ol\"a for many interesting discussions. The second author was supported by the ERC CG CRITICAL. The last author acknowledges support from the ERC SG Constamis, the NCCR SwissMAP, the Blavatnik Family Foundation and the Latsis Foundation.

\bibliography{main_NIPS}
\bibliographystyle{abbrv}

\appendix
\setcounter{prop}{0}
\setcounter{thm}{0}
\section{Appendix} \label{sec:proofs}
This appendix is dedicated to proving the key results of this paper, namely Proposition \ref{prop:output_limit} and Theorems \ref{thm:convergence_kernel_initialization} and \ref{thm:conv-ntk-training}, which describe the asymptotics of neural networks at initialization and during training. 

We study the limit of the NTK as $n_1, ..., n_{L-1} \to \infty$ sequentially, i.e. we first take $n_1 \to \infty$, then $n_2 \to \infty $, etc. This leads to much simpler proofs, but our results could in principle be strengthened to the more general setting when $\min(n_1, ..., n_{L-1}) \to \infty$.

A natural choice of convergence to study the NTK is with respect to the operator norm on kernels:
$$
\lVert K \rVert_{op} = \max_{\lVert f \rVert_{p^{in}} \leq 1} \lVert f \rVert_K = \max_{\lVert f \rVert_{p^{in}} \leq 1}\sqrt {\mathbb{E}_{x, x'}[f(x)^T K(x, x')f(x')]},
$$
where the expectation is taken over two independent $x, x' \sim p^{in}$. This norm depends on the input distribution $p^{in}$. In our setting, $p^{in}$ is taken to be the empirical measure of a finite dataset of distinct samples $x_1, ..., x_N$. As a result, the operator norm of $K$ is equal to the leading eigenvalue of the $N n_L \times N n_L$ Gram matrix $\left( K_{kk'}(x_i, x_j)\right)_{k, k' < n_L, i, j < N}$. In our setting, convergence in operator norm is hence equivalent to pointwise convergence of $ K $ on the dataset.

\subsection{Asymptotics at Initialization}
It has already been observed \citep{Neal1996, Lee2017} that the output functions $f_{\theta, i}$ for $i=1, ..., n_L$ tend to iid Gaussian processes in the infinite-width limit. 

\begin{prop}\label{prop:output_limit}
For a network of depth $L$ at initialization, with a Lipschitz nonlinearity $\sigma$, and in the limit as $n_1, ..., n_{L-1} \to \infty$ sequentially, the output functions $f_{\theta, k}$, for $k=1, ..., n_L$, tend  (in law) to iid centered Gaussian processes of covariance $\Sigma^{(L)}$, where $\Sigma^{(L)}$ is defined recursively by:
\begin{align*}
\Sigma^{(1)}(x, x') &= \frac{1}{n_0} x^T x' + \beta^2 \\
\Sigma^{(L+1)}(x, x') &= \mathbb{E}_{f}[\sigma(f(x)) \sigma(f(x'))] + \beta^2,
\end{align*}
taking the expectation with respect to a centered Gaussian process $f$ of covariance $\Sigma^{(L)}$.
\end{prop}

\begin{proof}We prove the result by induction. When $L=1$, there are no hidden layers and $f_\theta$ is a random affine function of the form:
$$
 f_\theta(x) = \frac{1}{\sqrt{n_0}} W^{(0)} x + \beta b^{(0)}.
$$
All output functions $f_{\theta, k}$ are hence independent and have covariance $\Sigma^{(1)}$ as needed.

The key to the induction step is to consider an $(L+1)$-network as the following composition: an $L$-network $\mathbb{R}^{n_0} \to \mathbb{R}^{n_L}$ mapping the input to the pre-activations $\tilde{\alpha}^{(L)}_i$, followed by an elementwise application of the nonlinearity $\sigma$ and then a random affine map $\mathbb{R}^{n_L} \to \mathbb{R}^{n_{L+1}}$. The induction hypothesis gives that in the limit as sequentially $n_1, ..., n_{L-1} \to \infty$ the preactivations $\tilde{\alpha}^{(L)}_i$ tend to iid Gaussian processes with covariance $\Sigma^{(L)}$. The outputs
$$
f_{\theta, i} = \frac{1}{\sqrt{n_L}} W_i^{(L)} \alpha^{(L)} + \beta b_i^{(L)}
$$
conditioned on the values of $\alpha^{(L)}$ are iid centered Gaussians with covariance
$$
 \tilde{\Sigma}^{(L+1)}(x, x') = \frac{1}{n_L} \alpha^{(L)}(x;\theta)^T \alpha^{(L)}(x';\theta) + \beta^2.
$$
By the law of large numbers, as $n_L \to \infty$, this covariance tends in probability to the expectation
$$\tilde{\Sigma}^{(L+1)}(x, x') \to \Sigma^{(L+1)}(x, x') = \mathbb{E}_{f \sim \mathcal{N}(0, \Sigma^{(L)})}[\sigma(f(x)) \sigma(f(x'))] + \beta^2.$$
In particular the covariance is deterministic and hence independent of $\alpha^{(L)}$. As a consequence, the conditioned and unconditioned distributions of $f_{\theta, i}$ are equal in the limit: they are iid centered Gaussian of covariance $\Sigma^{(L+1)}$.
\end{proof}

In the infinite-width limit, the neural tangent kernel, which is random at initialization, converges in probability to a deterministic limit.
\begin{thm}\label{thm:convergence_kernel_initialization}
For a network of depth $L$ at initialization, with a Lipschitz nonlinearity $\sigma$, and in the limit as the layers width $n_1, ..., n_{L-1} \to \infty$ sequentially, the NTK $\Theta^{(L)}$ converges in probability to a deterministic limiting kernel: $$\Theta^{(L)} \to \Theta^{(L)}_\infty \otimes Id_{n_L}.$$
The scalar kernel $\Theta^{(L)}_\infty : \mathbb{R}^{n_0} \times \mathbb{R}^{n_0} \to \mathbb{R}$ is defined recursively by
\begin{align*}
    \Theta^{(1)}_\infty(x, x') &= \Sigma^{(1)}(x, x') \\
    \Theta^{(L+1)}_\infty(x, x') &=  \Theta^{(L)}_\infty(x, x') \dot{\Sigma}^{(L+1)}(x, x')  + \Sigma^{(L+1)}(x, x'),
\end{align*}
where
\[
	\dot{\Sigma}^{(L+1)}\left(x,x'\right)=
	\mathbb{E}_{f\sim\mathcal{N}\left(0,\Sigma^{\left(L\right)}\right)}\left[\dot{\sigma}\left(f\left(x\right)\right)\dot{\sigma}\left(f\left(x'\right)\right)\right],
\]
taking the expectation with respect to a centered Gaussian process $f$ of covariance $\Sigma^{(L)}$, and where $\dot{\sigma}$ denotes the derivative of $\sigma$.
\end{thm}

\begin{proof}The proof is again by induction. When $L=1$, there is no hidden layer and therefore no limit to be taken. The neural tangent kernel is a sum over the entries of $W^{(0)}$ and those of $b^{(0)}$:
\begin{align*}
 \Theta_{kk'}(x, x') &= \frac{1}{n_0} \sum_{i=1}^{n_0} \sum_{j=1}^{n_1} x_i x'_i \delta_{jk}\delta_{jk'}  + \beta^2 \sum_{j=1}^{n_1} \delta_{jk}\delta_{jk'} \\
  &= \frac{1}{n_0} x^T x' \delta_{kk'} + \beta^2 \delta_{kk'} = \Sigma^{(1)}(x, x') \delta_{kk'}.
\end{align*}
 
Here again, the key to prove the induction step is the observation that a network of depth $L+1$ is an $L$-network mapping the inputs $x$ to the preactivations of the $L$-th layer $\tilde{\alpha}^{(L)}(x)$ followed by a nonlinearity and a random affine function. For a network of depth $ L + 1 $, let us therefore split the parameters into the parameters $\tilde{\theta}$ of the first $L$ layers and those of the last layer $(W^{(L)}, b^{(L)})$. 
 
By Proposition \ref{prop:output_limit} and the induction hypothesis, as $n_1, ..., n_{L-1} \to \infty$ the pre-activations $\tilde{\alpha}^{(L)}_i$ are iid centered  Gaussian with covariance $\Sigma^{(L)}$ and the neural tangent kernel $\Theta^{(L)}_{ii'}(x, x')$ of the smaller network converges to a deterministic limit:
$$
  \left(\partial_{\tilde{\theta}} \tilde{\alpha}^{(L)}_{i}(x;\theta)\right)^T \partial_{\tilde{\theta}} \tilde{\alpha}^{(L)}_{i'}(x';\theta) \to \Theta^{(L)}_\infty(x, x') \delta_{ii'}.
$$

We can split the neural tangent network into a sum over the parameters $\tilde{\theta}$ of the first $L$ layers and the remaining parameters $W^{(L)}$ and $b^{(L)}$. 

For the first sum let us observe that by the chain rule:
$$
\partial_{\tilde{\theta}_p} f_{\theta, k}(x) = \frac{1}{\sqrt{n_L}} \sum_{i=1}^{n_L} \partial_{\tilde{\theta}_p} \tilde{\alpha}^{(L)}_{i}(x;\theta) \dot{\sigma}(\tilde{\alpha}^{(L)}_i(x;\theta)) W^{(L)}_{ik}.
$$
By the induction hypothesis, the contribution of the parameters $\tilde{\theta}$ to the neural tangent kernel $\Theta^{(L+1)}_{kk'}(x, x')$ therefore converges as $n_1, ..., n_{L-1} \to \infty$:
\begin{align*}
    &\frac{1}{n_L}\! \sum_{i, i'=1}^{n_L}\! \Theta^{(L)}_{ii'}(x, x') \dot{\sigma}\!\left(\!\tilde{\alpha}^{(L)}_i(x;\theta)\!\right) \dot{\sigma}\!\left(\!\tilde{\alpha}^{(L)}_{i'}(x';\theta)\!\right) W^{(L)}_{ik}W^{(L)}_{i'k'} \\
    &\!\to \frac{1}{n_L}\!\sum_{i=1}^{n_L}\! \Theta^{(L)}_\infty(x, x') \dot{\sigma}\!\left(\!\tilde{\alpha}^{(L)}_i(x;\theta)\!\right) \dot{\sigma}\!\left(\!\tilde{\alpha}^{(L)}_i(x';\theta)\!\right) W^{(L)}_{ik}W^{(L)}_{ik'}
\end{align*}
By the law of large numbers, as $n_L \to \infty$, this tends to its expectation which is equal to
$$
\Theta^{(L)}_\infty(x, x') \dot{\Sigma}^{(L+1)}(x, x') \delta_{kk'}.
$$
It is then easy to see that the second part of the neural tangent kernel, the sum over $W^{(L)}$ and $b^{(L)}$ converges to $\Sigma^{(L+1)} \delta_{kk'}$ as $n_1, ..., n_L \to \infty$.
\end{proof}

\subsection{Asymptotics during Training}
Given a training direction $t \mapsto d_t \in \mathcal{F}$, a neural network is trained in the following manner: the parameters $\theta_p$ are initialized as iid $\mathcal{N}(0, 1)$ and follow the differential equation:
$$
\partial_t \theta_p(t) = \left< \partial_{\theta_p} F^{(L)}, d_t \right>_{p^{in}}.
$$
In this context, in the infinite-width limit, the NTK stays constant during training:

\begin{thm}\label{thm:conv-ntk-training}
Assume that $ \sigma $ is a Lipschitz, twice differentiable nonlinearity function, with bounded second derivative. For any $T$ such that the integral $\int_0^T \| d_t \|_{p^{in}} dt$ stays stochastically bounded, as  $n_1, ..., n_{L-1} \to \infty$ sequentially, we have, uniformly for $t\in[0, T]$,
$$
\Theta^{(L)}(t) \to \Theta^{(L)}_\infty \otimes Id_{n_L}.
$$
As a consequence, in this limit, the dynamics of $f_\theta$ is described by the differential equation
\begin{align*}
    \partial_t f_{\theta(t)} = \Phi_{\Theta^{(L)}_\infty \otimes Id_{n_L}} \left( \left<d_{t}, \cdot \right>_{p^{in}} \right).
\end{align*}
\end{thm}

\begin{proof}
As in the previous theorem, the proof is by induction on the depth of the network. When $L=1$, the neural tangent kernel does not depend on the parameters, it is therefore constant during training.

For the induction step, we again split an $L+1$ network into a network of depth $L$ with parameters $\tilde{\theta}$ and top layer connection weights $W^{(L)}$ and bias $b^{(L)}$. The smaller network follows the training direction
$$
 d'_{t} =  \dot{\sigma}\left(\tilde{\alpha}^{(L)}(t)\right) \left(\frac{1}{\sqrt{n_L}}W^{(L)}(t)\right)^T d_t \label{eq:direct-small-network}
$$
for $i=1, \ldots, n_L$, where the function $\tilde{\alpha}^{(L)}_i(t) $ is defined as $ \tilde{\alpha}^{(L)}_i(\cdot ; \theta(t))$. We now want to apply the induction hypothesis to the smaller network. For this, we need to show that
$ \int_{0}^{T} \lVert d'_t \rVert_{p^{in}} \mathrm{d} t $ is stochastically bounded as $ n_1, \ldots, n_L \to \infty $.  Since $ \sigma $ is a $c$-Lipschitz function, we have
that 
\[
	\lVert d'_t \rVert_{p^{in}} \leq c \lVert  \frac{1}{\sqrt{n_L}} W^{(L)} (t) \rVert_{op}
	\lVert d_t \rVert_{p^{in}}.
\] 
To apply the induction hypothesis, we now need to bound $ \lVert  \frac{1}{\sqrt{n_L}} W^{(L)} (t) \rVert_{op} $. For this, we use the following lemma, which is proven in Appendix \ref{Appendix-3} below:

\begin{lem} \label{lem:control-w}
With the setting of Theorem \ref{thm:conv-ntk-training}, for a network of depth $ L + 1$, for any $ \ell =1, \ldots, L $, we have the convergence in probability:  \[ 
\lim_{n_L \to \infty} \cdots \lim_{n_1 \to \infty} 	
\sup_{t \in [0, T]} \lVert \frac{1}{\sqrt{n_\ell}} \left( W^{(\ell)}(t) - W^{(\ell)}(0) \right) \rVert_{op} = 0
\]
\end{lem}

From this lemma, to bound  $ \lVert  \frac{1}{\sqrt{n_L}} W^{(L)} (t) \rVert_{op} $, it is hence enough to bound $ \lVert  \frac{1}{\sqrt{n_L}} W^{(L)} (0) \rVert_{op} $. From the law of large numbers, we obtain that the norm of each of the $ n_{L + 1} $ rows of $ W^{(L)} (0) $ is bounded, and hence that $ \lVert  \frac{1}{\sqrt{n_L}} W^{(L)} (0) \rVert_{op} $ is bounded (keep in mind that $ n_{L + 1} $ is fixed, while $ n_1, \ldots, n_L $ grow).

From the above considerations, we can apply the induction hypothesis to the smaller network, yielding, in the limit as $n_1, \ldots, n_L \to \infty$ (sequentially), that the dynamics is governed by the constant kernel $ \Theta^{(L)}_\infty $:
\[
    \partial_t \tilde{\alpha}^{(L)}_{i}(t) =  \frac{1}{\sqrt{n_L}} \Phi_{\Theta^{(L)}_\infty} \left( \left< \dot{\sigma}\left(\tilde{\alpha}^{(L)}_i(t)\right) \left(W^{(L)}_i(t)\right)^T d_t, \cdot \right>_{p^{in}} \right) .
\] 
At the same time, the parameters of the last layer evolve according to
\begin{align*}
 \partial_t W^{(L)}_{ij}(t) &= \frac{1}{\sqrt{n_L}} \left< \alpha^{(L)}_i(t),  d_{t, j} \right>_{p^{in}}.
\end{align*}

We want to give an upper bound on the variation of the weights columns $W_i^{(L)}(t)$ and of the activations $\tilde{\alpha}^{(L)}_{i}(t)$ during training in terms of $L^2$-norm and $p^{in}$-norm respectively. Applying the Cauchy-Schwarz inequality for each $ j $, summing and using $ \partial_t || \cdot || \leq || \partial_t \cdot || $), we have
\begin{align*}
    \partial_t \left\lVert W^{(L)}_{i}(t) - W^{(L)}_{i}(0) \right\rVert_2 &\leq \frac{1}{\sqrt{n_L}} ||\alpha^{(L)}_i(t)||_{p^{in}} ||d_t||_{p^{in}}. 
\end{align*}

Now, observing that the operator norm of $\Phi_{\Theta_\infty^{(L)}}$ is equal to $\vert \vert \Theta_\infty^{(L)}\vert \vert_{op}$, defined in the introduction of Appendix \ref{sec:proofs}, and using the Cauchy-Schwarz inequality, we get
\begin{align*}
    \partial_t \left\lVert \tilde{\alpha}^{(L)}_{i}(t) - \tilde{\alpha}^{(L)}_{i}(0) \right\rVert_{p^{in}} &\leq \frac{1}{\sqrt{n_L}} \left\lVert \Theta^{(L)}_\infty \right\rVert_{op} \left\lVert \dot{\sigma}\left(\tilde{\alpha}^{(L)}_i(t)\right) \right\rVert_\infty \left\lVert W^{(L)}_i(t)\right\rVert_2 \left\lVert d_t \right\rVert_{p^{in}},
\end{align*}
where the sup norm $ \lVert \cdot \rVert_{\infty} $ is defined by $\left\lVert f \right\lVert_\infty = \sup_x | f(x) |.$

To bound both quantities simultaneously, study the derivative of the quantity
$$
A(t) = ||\alpha^{(L)}_i(0)||_{p^{in}} + c \left\lVert \tilde{\alpha}^{(L)}_{i}(t) - \tilde{\alpha}^{(L)}_{i}(0) \right\rVert_{p^{in}} + || W^{(L)}_i(0) ||_2 + \left\lVert W^{(L)}_{i}(t) - W^{(L)}_{i}(0) \right\rVert_2.
$$
We have
\begin{align*}
 \partial_t A(t) &\leq \frac{1}{\sqrt{n_L}} \left( c^2 \left\lVert \Theta^{(L)}_\infty \right\rVert_{op} \left\lVert W^{(L)}_i(t)\right\rVert_2 + ||\alpha^{(L)}_i(t)||_{p^{in}} \right) ||d_t||_{p^{in}} \\ 
    &\leq \frac{\max\{c^2 \|\Theta^{(L)}_\infty \|_{op} , 1\} }{\sqrt{n_L}}\|d_t \|_{p^{in}}  A(t),
\end{align*} 
where, in the first inequality, we have used that $ | \dot{\sigma} | \leq c $ and, in the second inequality, that the sum $\lVert W^{(L)}_i(t) \rVert_2 + ||\alpha^{(L)}_i(t)||_{p^{in}}$ is bounded by $A(t)$.
Applying Gr\"onwall's Lemma, we now get
$$
 A(t) \leq A(0) \exp\left(\frac{\max\{c^2 \|\Theta^{(L)}_\infty \|_{op} , 1\} }{\sqrt{n_L}} \int_0^t \|d_s \|_{p^{in}} ds\right).
$$
Note that $ \|\Theta^{(L)}_\infty \|_{op}$ is constant during training. Clearly the value inside of the exponential converges to zero in probability as $n_L \to \infty$ given that the integral $\int_0^t \|d_t \|_{p^{in}} ds$ stays stochastically bounded. The variations of the activations $\left\lVert \tilde{\alpha}^{(L)}_{i}(t) - \tilde{\alpha}^{(L)}_{i}(0) \right\rVert_{p^{in}}$ and weights $\left\lVert W^{(L)}_{i}(t) - W^{(L)}_{i}(0) \right\rVert_2$ are bounded by $c^{-1}(A(t) - A(0))$ and $A(t) - A(0)$ respectively, which converge to zero at rate $O\left(\frac{1}{\sqrt{n_L}}\right)$. 

We can now use these bounds to control the variation of the NTK and to prove the theorem. To understand how the NTK evolves, we study the evolution of the derivatives with respect to the parameters. The derivatives with respect to the bias parameters of the top layer $\partial_{b^{(L)}_j}f_{\theta, j'}$ are always equal to $\delta_{jj'}$. The derivatives with respect to the connection weights of the top layer are given by
$$
\partial_{W^{(L)}_{ij}}f_{\theta, j'}(x) = \frac{1}{\sqrt{n_L}} \alpha^{(L)}_i(x ;\theta) \delta_{jj'}.
$$
The pre-activations $\tilde{\alpha}^{(L)}_i$ evolve at a rate of $\frac{1}{\sqrt{n_L}}$ and so do the activations $\alpha^{(L)}_i$. The summands $\partial_{W^{(L)}_{ij}}f_{\theta, j'}(x) \otimes \partial_{W^{(L)}_{ij}}f_{\theta, j''}(x')$ of the NTK hence vary at rate of $n_L^{-3/2}$ which induces a variation of the NTK of rate $\frac{1}{\sqrt{n_L}}$.

Finally let us study the derivatives with respect to the parameters of the lower layers 
$$
\partial_{\tilde{\theta}_k} f_{\theta, j}(x) = \frac{1}{\sqrt{n_L}} \sum_{i=1}^{n_L} \partial_{\tilde{\theta}_k} \tilde{\alpha}^{(L)}_{i}(x;\theta) \dot{\sigma}\left(\tilde{\alpha}^{(L)}_i(x;\theta)\right) W^{(L)}_{ij}.
$$
Their contribution to the NTK $\Theta^{(L+1)}_{jj'}(x, x')$ is
\begin{align*}
    &\frac{1}{n_L}\! \sum_{i, i'=1}^{n_L}\! \Theta^{(L)}_{ii'}(x, x') \dot{\sigma}\!\left(\!\tilde{\alpha}^{(L)}_i(x;\theta)\!\right) \dot{\sigma}\!\left(\!\tilde{\alpha}^{(L)}_{i'}(x';\theta)\!\right) W^{(L)}_{ij}W^{(L)}_{i'j'}.
\end{align*}
By the induction hypothesis, the NTK of the smaller network $\Theta^{(L)}$ tends to $\Theta^{(L)}_\infty \delta_{ii'}$ as $n_1, ..., n_{L-1} \to \infty$. The contribution therefore becomes
\begin{align*}
    &\frac{1}{n_L}\! \sum_{i=1}^{n_L}\! \Theta^{(L)}_\infty(x, x') \dot{\sigma}\!\left(\!\tilde{\alpha}^{(L)}_i(x;\theta)\!\right) \dot{\sigma}\!\left(\!\tilde{\alpha}^{(L)}_{i}(x';\theta)\!\right) W^{(L)}_{ij}W^{(L)}_{ij'}.
\end{align*}
The connection weights $W^{(L)}_{ij}$ vary at rate $\frac{1}{\sqrt{n_L}}$, inducing a change of the same rate to the whole sum. We simply have to prove that the values $\dot{\sigma}(\tilde{\alpha}^{(L)}_i(x;\theta))$ also change at rate $\frac{1}{\sqrt{n_L}}$. Since the second derivative of $ \sigma $ is bounded, we have that 
\[
 \partial_t \left( \dot{\sigma}\left(\tilde{\alpha}^{(L)}_i(x;\theta(t))\right)  \right) = O\left( \partial_t \tilde{\alpha}^{(L)}_i(x;\theta(t)) \right).
\]
Since $  \partial_t \tilde{\alpha}^{(L)}_i(x;\theta(t)) $ goes to zero at a rate $\frac{1}{\sqrt{n_L}}$ by the bound on $ A(t) $ above, this concludes the proof.



\end{proof}

It is somewhat counterintuitive that the variation of the activations of the hidden layers $\alpha^{(\ell)}_i$ during training goes to zero as the width becomes large\footnote{As a consequence, the pre-activations stay Gaussian during training as well, with the same covariance $\Sigma^{(\ell)}$.}. It is generally assumed that the purpose of the activations of the hidden layers is to learn ``good'' representations of the data during training. However note that even though the variation of each individual activation shrinks, the number of neurons grows, resulting in a significant collective effect. This explains why the training of the parameters of each layer $\ell$ has an influence on the network function $f_\theta$ even though it has asymptotically no influence on the individual activations of the layers $\ell'$ for $\ell<\ell'<L$.

\subsection{A Priori Control during Training} \label{Appendix-3}
The goal of this section is to prove Lemma \ref{lem:control-w}, which is a key ingredient in the proof of Theorem \ref{thm:conv-ntk-training}. Let us first recall it:
\setcounter{lem}{0}
\begin{lem} \label{lem:control-w}
With the setting of Theorem \ref{thm:conv-ntk-training}, for a network of depth $ L + 1$, for any $ \ell =1, \ldots, L $, we have the convergence in probability:  \[ 
\lim_{n_L \to \infty} \cdots \lim_{n_1 \to \infty} 	
\sup_{t \in [0, T]} \lVert \frac{1}{\sqrt{n_\ell}} \left( W^{(\ell)}(t) - W^{(\ell)}(0) \right) \rVert_{op} = 0
\]
\end{lem}

\begin{proof}
We prove the lemma for all $ \ell =1, \ldots, L $ simultaneously, by expressing the variation of the weights $ \frac{1}{\sqrt{n_\ell}} W^{(\ell)} $ and activations $ \frac{1}{\sqrt{n_\ell}} \tilde{\alpha}^{(\ell)} $ in terms of `back-propagated' training directions $ d^{(1)}, \ldots, d^{(L)} $ associated with the lower layers and the NTKs of the corresponding subnetworks:    
\begin{enumerate}
	 \item At all times, the evolution of the preactivations and weights is given by:
\begin{align*}
\partial_{t}\tilde{\alpha}^{(\ell)} & =\Phi_{\Theta^{(\ell)}} \left( <d_{t}^{(\ell)},\cdot>_{p^{in}} \right) \\
\partial_{t}W^{(\ell)} & =\frac{1}{\sqrt{n_{\ell}}}<\alpha^{(\ell)},d_{t}^{(\ell+1)}>_{p^{in}}, 
\end{align*}
where the layer-wise training directions $ d^{(1)}, \ldots, d^{(L)} $ are defined recursively by 
\begin{align*}
d_{t}^{\left(\ell\right)} & =\begin{cases}
d_{t} & \text{ if }\ell=L+1\\
\dot{\sigma}\left(\tilde{\alpha}^{\left(\ell\right)}\right)\left(\frac{1}{\sqrt{n_{\ell}}}W^{\left(\ell\right)}\right)^{T}d_{t}^{\left(\ell+1\right)} & \text{ if }\ell\leq L,
\end{cases}
\end{align*}
and where the sub-network NTKs $ \Theta^{(\ell)}$ satisfy 
\begin{align*}
\Theta^{(1)} & =\left[\left[\frac{1}{\sqrt{n_{0}}}\alpha^{(0)}\right]^{T}\left[\frac{1}{\sqrt{n_{0}}}\alpha^{(0)}\right]\right]\otimes Id_{n_{\ell}}+\beta^{2}\otimes Id_{n_{\ell}}\\
\Theta^{(\ell+1)} & =\frac{1}{\sqrt{n_{\ell}}}W^{(\ell)}\dot{\sigma}(\tilde{\alpha}^{(\ell)})\Theta^{(\ell)}\dot{\sigma}(\tilde{\alpha}^{(\ell)})\frac{1}{\sqrt{n_{\ell}}}W^{(\ell)}\\
 & +\left[\left[\frac{1}{\sqrt{n_{\ell}}}\alpha^{(\ell)}\right]^{T}\left[\frac{1}{\sqrt{n_{\ell}}}\alpha^{(\ell)}\right]\right]\otimes Id_{n_{\ell}}+\beta^{2}\otimes Id_{n_{\ell}}.
\end{align*}
\item
Set $ w^{(k)} (t) := \left\Vert \frac{1}{\sqrt{n_{k}}}W^{(k)} (t)\right\Vert _{op} $ and
$a^{(k)}\left(t\right):=\left\Vert \frac{1}{\sqrt{n_{k}}}\alpha^{\left(k\right)}\left(t\right)\right\Vert _{p^{in}}$.
The identities of the previous step yield the following recursive bounds:
\[
\left\Vert d_{t}^{(\ell)}\right\Vert _{p^{in}}\le c w^{(\ell)}(t)\left\Vert d_{t}^{(\ell+1)}\right\Vert _{p^{in}},
\]
where $c$ is the Lipschitz constant of $\sigma$. These bounds lead to
\[
\left\Vert d_{t}^{(\ell)}\right\Vert _{p^{in}}\leq c^{L+1-\ell}\prod_{k=\ell}^{L}w^{(k)}(t)\left\Vert d_{t}\right\Vert _{p^{in}}.
\]
For the subnetworks NTKs we have the recursive bounds
\begin{align*}
\|\Theta^{(1)}\|_{op} & \le(a^{(0)}(t))^{2}+\beta^{2}.\\
\|\Theta^{(\ell+1)}\|_{op} & \le c^{2} (w^{(\ell)}(t) )^2
\| \Theta^{(\ell)} \|_{op}+(a^{(\ell)}(t))^{2}+\beta^{2},
\end{align*}
which lead to 
\[
\|\Theta^{(\ell+1)}\|_{op}\leq\mathcal{P}\left(a^{(1)},\ldots,a^{(\ell)},w^{(1)},\ldots,w^{(\ell)}\right),
\]
where $ \mathcal P $ is a polynomial which only depends on $ \ell, c, \beta $ and $ p^{in} $.
\item Set 
\begin{align*}
\tilde{a}^{(k)}\left(t\right) & :=\left\Vert \frac{1}{\sqrt{n_{k}}}\left(\tilde{\alpha}^{\left(k\right)}\left(t\right)-\tilde{\alpha}^{\left(k\right)}\left(0\right)\right)\right\Vert _{p^{in}}\\
\tilde{w}^{(k)}\left(t\right) & :=\left\Vert \frac{1}{\sqrt{n_{k}}}\left(W^{\left(k\right)}\left(t\right)-W^{\left(k\right)}\left(0\right)\right)\right\Vert _{op}
\end{align*}
 and define 
\begin{align*}
A\left(t\right)=\sum_{k=1}^{L}a^{\left(k\right)}\left(0\right)+c\tilde{a}^{\left(k\right)}\left(t\right)+w^{\left(k\right)}\left(0\right)+\tilde{w}^{\left(k\right)}\left(t\right).
\end{align*}
Since $a^{\left(k\right)}\left(t\right)\leq a^{\left(k\right)}\left(0\right)+c\tilde{a}^{\left(k\right)}\left(t\right)$
and $w^{\left(k\right)}\left(t\right)\leq w^{\left(k\right)}\left(0\right)+\tilde{w}^{\left(k\right)}\left(t\right)$,
controlling $A\left(t\right)$ will enable us to control the $a^{\left(k\right)}\left(t\right)$
and $w^{\left(k\right)}\left(t\right)$. Using the formula at the
beginning of the first step, we obtain
\begin{align*}
\partial_{t}\tilde{a}^{\left(\ell\right)}\left(t\right) & \le\frac{1}{\sqrt{n_{\ell}}}\|\Theta^{(\ell)}(t)\|_{op}\|d_{t}^{(\ell)}\|_{p^{in}}\\
\partial_{t}\tilde{w}^{\left(\ell\right)}\left(t\right) & \le\frac{1}{\sqrt{n_{\ell}}}a^{\left(\ell\right)}\left(t\right)\|d_{t}^{(\ell+1)}\|_{p^{in}}.
\end{align*}
This allows one to bound the derivative of $A\left(t\right)$ as follows:
\[
\partial_{t}A\left(t\right)\le\sum_{\ell=1}^{L}\frac{c}{\sqrt{n_{\ell}}}\|\Theta^{(\ell)}(t)\|_{op}\|d_{t}^{(\ell)}\|_{p^{in}}+\frac{1}{\sqrt{n_{\ell}}}a^{\left(\ell\right)}\left(t\right)\|d_{t}^{(\ell+1)}\|_{p^{in}}.
\]
Using the polynomial bounds on $\|\Theta^{(\ell)}(t)\|_{op}$ and
$\|d_{t}^{(\ell+1)}\|_{p^{in}}$ in terms of the $a^{\left(k\right)}$
and $w^{\left(k\right)}$ for $k=1,\ldots\ell$ obtained in the previous
step, we get that 
\[
\text{\ensuremath{\partial_{t}A\left(t\right)\leq\frac{1}{\sqrt{\min\left\{ n_{1},\ldots,n_{L}\right\} }}\mathcal{Q}\left(w^{\left(1\right)}\left(t\right),\ldots,w^{\left(L\right)}\left(t\right),a^{\left(1\right)}\left(t\right),\ldots,a^{\left(L\right)}\left(t\right)\right)\|d_{t}\|_{p^{in}},}}
\]
where the polynomial $Q$ only depends on $L,c,\beta$ and $p^{in}$
and has positive coefficients. As a result, we can use $a^{\left(k\right)}\left(t\right)\leq a^{\left(k\right)}\left(0\right)+c\tilde{a}^{\left(k\right)}\left(t\right)$
and $w^{\left(k\right)}\left(t\right)\leq w^{\left(k\right)}\left(0\right)+\tilde{w}^{\left(k\right)}\left(t\right)$
to get the polynomial bound
\[
\partial_{t}A\left(t\right)\leq\frac{1}{\sqrt{\min\left\{ n_{1},\ldots,n_{L}\right\} }}\tilde{\mathcal{Q}}\left(A\left(t\right)\right)\|d_{t}\|_{p^{in}}.
\]

\item
Let us now observe that $A\left(0\right)$ is stochastically bounded
as we take the sequential limit $\lim_{n_{L}\to\infty}\cdots\lim_{n_{1}\to\infty}$
as in the statement of the lemma. In this limit, we indeed have that
$w^{\left(\ell\right)}$and $a^{\left(\ell\right)}$ are convergent: 
we have $w^{\left(\ell\right)}\to0$, while $a^{\left(\ell\right)}$ converges
by Proposition \ref{prop:output_limit}. 

The polynomial control we obtained on the derivative of $A\left(t\right)$
now allows one to use (a nonlinear form of, see e.g. \cite{dragomir}) Grönwall's Lemma: we obtain that $A\left(t\right)$
stays uniformly bounded on $\left[0,\tau\right]$ for some $\tau=\tau\left(n_{1},\ldots,n_{L}\right)>0$,
and that $\tau\to T$ as $\min\left(n_{1},\ldots,n_{L}\right)\to\infty$,
owing to the $\frac{1}{\sqrt{\min\left\{ 1,\ldots,n_{L}\right\} }}$in
front of the polynomial. Since $A\left(t\right)$ is bounded, the
differential bound on $A\left(t\right)$ gives that the derivative
$\partial_{t}A\left(t\right)$ converges uniformly to $0$ on $\left[0,\tau\right]$
for any $\tau<T$, and hence $A\left(t\right)\to A\left(0\right)$.
This concludes the proof of the lemma.

\end{enumerate}
\end{proof}

\subsection{Positive-Definiteness of $ \Theta_\infty^{(L)} $} \label{Appendix-4}
This subsection is devoted to the proof of Proposition \ref{prop:pos-def}, which we now recall:
\begin{prop}
\label{prop:pos-def}For a non-polynomial Lipschitz nonlinearity $\sigma$,
for any input dimension $n_{0}$, the restriction of the limiting
NTK $\Theta_{\infty}^{(L)}$ to the unit sphere $\mathbb{S}^{n_{0}-1}=\{x\in\mathbb{R}^{n_{0}}:x^{T}x=1\}$
is positive-definite if $L\geq2$. 
\end{prop}

A key ingredient for the proof of Proposition \ref{prop:pos-def}
is the following Lemma, which comes from \cite{Daniely}.
\begin{lem}[Lemma 12(a) in suppl. mat. of \cite{Daniely}]
\label{lem:daniely}Let $\hat{\mu}:\left[-1,1\right]\to\mathbb{R}$
denote the dual of a Lipschitz function $\mu:\mathbb{R}\to\mathbb{R}$,
defined by $\hat{\mu}\left(\rho\right)=\mathbb{E}_{\left(X,Y\right)}\left[\mu\left(X\right)\mu\left(Y\right)\right]$
where $\left(X,Y\right)$ is a centered Gaussian vector of covariance
$\Sigma$, with 
\[
\Sigma=\begin{pmatrix}1 & \rho\\
\rho & 1
\end{pmatrix}.
\]
If the expansion of $\mu$ in Hermite polynomials $\left(h_{i}\right)_{i\geq0}$
is given by $\mu=\sum_{i=0}^{\infty}a_{i}h_{i}$, we have 
\[
	\hat{\mu}\left(\rho\right)=\sum_{i=0}^{\infty}a_{i}^{2}\rho^{i}.
\] 
\end{lem}

The other key ingredient for proving Proposition \ref{prop:pos-def}
is the following theorem, which is a slight reformulation of Theorem
1(b) in \cite{Gneiting}, which itself is a generalization of a classical result
of Sch\"onberg:
\begin{thm}
\label{thm:schoenberg}For a function $f:\text{\ensuremath{\left[-1,1\right]}}\to\mathbb{R}$ with $f\left(\rho\right)=\sum_{n=0}^{\infty}b_{n}\rho^{n}$,
the kernel $K_{f}^{\left(n_{0}\right)}:\mathbb{S}^{n_{0}-1}\times\mathbb{S}^{n_{0}-1}\to\mathbb{R}$
defined by
\[
K_{f}^{\left(n_{0}\right)}\left(x,x'\right)=f\left(x^{T}x'\right)
\]
 is positive-definite for any $n_{0}\geq1$ if and only if 
the coefficients $b_{n}$ are strictly positive for infinitely
many even and infinitely many odd integers $n$. 
\end{thm}

With Lemma \ref{lem:daniely} and Theorem \ref{thm:schoenberg} above,
we are now ready to prove Proposition \ref{prop:pos-def}.
\begin{proof}[Proof of Proposition \ref{prop:pos-def}]
We first decompose the limiting NTK $\Theta^{\left(L\right)}$ recursively,
relate its positive-definiteness to that of the activation kernels,
then show that the positive-definiteness of the activation kernels
at level $2$ implies that of the higher levels, and finally show the
positive-definiteness at level $2$ using Lemma \ref{lem:daniely}
and Theorem \ref{thm:schoenberg}:
\begin{enumerate}
\item Observe that for any $L\geq1$, using the notation of Theorem \ref{thm:convergence_NTK_initialization}, we have
\[
\Theta^{\left(L+1\right)}=\dot{\Sigma}^{\left(L\right)}\Theta^{\left(L\right)}+\Sigma^{\left(L+1\right)}.
\]
Note that the kernel $ \dot{\Sigma}^{\left(L\right)}\Theta^{\left(L\right)} $ is positive semi-definite, being the product of two positive semi-definite kernels.
Hence, if we show that $\Sigma^{\left(L+1\right)}$ is positive-definite,
this implies that $\Theta^{\left(L+1\right)}$ is positive-definite.
\item By definition, with the notation of Proposition \ref{prop:output_limit}
we have 
\[
\Sigma^{\left(L+1\right)}\left(x,x'\right)=\mathbb{E}_{f\sim\mathcal{N}\left(\text{0,}\Sigma^{\left(L\right)}\right)}\left[\sigma\left(f\left(x\right)\right)\sigma\left(f\left(x'\right)\right)\right]+\beta^{2}.
\]
This gives, for any collection of coefficients $c_{1},\ldots,c_{d}\in\mathbb{R}$
and any pairwise distinct $x_{1},\ldots,x_{d}\in\mathbb{R}^{n_{0}}$, that
\[
\sum_{i,j=1}^{d}c_{i}c_{j}\Sigma^{\left(L+1\right)}\left(x_{i},x_{j}\right)=\mathbb{E}\left[\left(\sum_{i}c_{i}\sigma\left(f\left(x_{i}\right)\right)\right)^{2}\right]+\left(\beta\sum_{i}c_{i}\right)^{2}.
\]
Hence the left-hand side only vanishes if $\sum c_{i}\sigma\left(f\left(x_{i}\right)\right)$
is almost surely zero. If $\Sigma^{\left(L\right)}$ is positive-definite,
the Gaussian $\left(f\left(x_{i}\right)\right)_{i=1,\ldots d}$ is
non-degenerate, so this only occurs when $c_{1}=\cdots=c_{d}=0$ since
$\sigma$ is assumed to be non-constant. This shows that the positive-definiteness
of $\Sigma^{\left(L+1\right)}$ is implied by that of $\Sigma^{\left(L\right)}$.
By induction, if $\Sigma^{\left(2\right)}$ is positive-definite,
we obtain that all $\Sigma^{\left(L\right)}$ with $L\geq2$ are positive-definite
as well. By the first step this hence implies that $\Theta^{\left(L\right)}$ is positive-definite
as well.
\item By the previous steps, to prove the proposition, it suffices to show
the positive-definitess of $\Sigma^{\left(2\right)}$ on the unit
sphere $\mathbb{S}^{n_{0}-1}$. We have
\[
\Sigma^{\left(2\right)}\left(x,x'\right)=\mathbb{E}_{\left(X,Y\right)\sim\mathcal{N}\left(0,\tilde{\Sigma}\right)}\left[\sigma\left(X\right)\sigma\left(Y\right)\right]+\beta^{2}
\]
where 
\[
\tilde{\Sigma}=\left(\begin{array}{cc}
\frac{1}{n_{0}}+\beta^{2} & \frac{1}{n_{0}}x^{T}x'+\beta^{2}\\
\frac{1}{n_{0}}x^{T}x+\beta^{2} & \frac{1}{n_{0}}+\beta^{2}
\end{array}\right).
\]
A change of variables then yields 
\begin{equation}
\mathbb{E}_{\left(X,Y\right)\sim\mathcal{N}\left(0,\tilde{\Sigma}\right)}\left[\sigma\left(X\right)\sigma\left(Y\right)\right]+\beta^{2}=\hat{\mu}\left(\frac{n_{0}\beta^{2}+x^{T}x'}{n_{0}\beta^{2}+1}\right)+\beta^{2},\label{eq:from-sigma-to-mu-hat}
\end{equation}
where $\hat{\mu}:\left[-1,1\right]\to\mathbb{R}$ is the dual in the
sense of Lemma \ref{lem:daniely} of the function $\mu:\mathbb{R}\to\mathbb{R}$
defined by $\mu\left(x\right)=\sigma\left(x\sqrt{\frac{1}{n_{0}}+\beta^{2}}\right)$. 
\item Writing the expansion of $\mu$ in Hermite polynomials $\left(h_{i}\right)_{i\geq0}$
\[
\mu=\sum_{i=0}^{\infty}a_{i}h_{i},
\]
we obtain that $\hat{\mu}$ is given by the power series
\[
\hat{\mu}\left(\rho\right)=\sum_{i=0}^{\infty}a_{i}^{2}\rho^{i},
\]
Since $\sigma$ is non-polynomial, so is $\mu$, and as a result,
there is an infinite number of nonzero $a_{i}$'s in the above sum. 
\item Using (\ref{eq:from-sigma-to-mu-hat}) above, we obtain that 
\[
\Sigma^{\left(2\right)}\left(x,x'\right)=\nu\left(x^{T}x'\right),
\]
where $\nu:\mathbb{R}\to\mathbb{R}$ is defined by 
\[
\nu\left(\rho\right)=\beta^{2}+\sum_{i=0}^{\infty}a_{i}\left(\frac{n_{0}\beta^{2}+\rho}{n_{0}\beta^{2}+1}\right)^{i},
\]
where the $a_{i}$'s are the coefficients of the Hermite expansion
of $\mu$. Now, observe that by the previous step, the power series
expansion of $\nu$ contains both an infinite number of nonzero even
terms and an infinite number of nonzero odd terms. This enables one
to apply Theorem \ref{thm:schoenberg} to obtain that $\Sigma^{\left(2\right)}$
is indeed positive-definite, thereby concluding the proof. 
\end{enumerate}
\end{proof}
\begin{rem}
Using similar techniques to the one applied in the proof above, one
can show a converse to Proposition \ref{prop:pos-def}: if the nonlinearity
$\sigma$ is a polynomial, the corresponding NTK $\Theta^{\left(2\right)}$
is not positive-definite $\mathbb{S}^{n_{0}-1}$ for certain input
dimensions $n_{0}$.
\end{rem}

\end{document}